\documentclass[12pt]{amsart}

\usepackage{amssymb, amsmath, amsthm,amsrefs,epsfig}
\usepackage{latexsym}
\usepackage{color}
\usepackage{exscale}
\usepackage{fullpage}
\usepackage{bm, bbm, mathrsfs}

\newcommand{\oset}{{\mathcal Z}}
\newcommand{\iset}{{\mathcal X}}

\newcommand{\sig}{{\mathcal S}}
\newcommand{\sigs}{{\mathfrak S}}

\newcommand{\set}{{ S}}
\newcommand{\impsig}{{\hat\sig}}
\newcommand{\jet}{{\mathcal J}}

\newcommand{\f}{{\mathcal F}}
\newcommand{\GL}{{{\mathcal GL}}}

\newcommand{\aff}{{\mathcal A}}
\newcommand{\saff}{{\mathcal {SA}}}
\newcommand{\pgl}{{\mathcal {PGL}}}
\newcommand{\affp}{{\mathcal A(2)}\times{\mathcal A(3)}}
\newcommand{\pta}{{\pgl(3)}\times{\mathcal A(3)}}
\newcommand{\act}{{ \Phi}}
\newcommand{\fp}{{\mathcal {FP}}}
\newcommand{\ap}{{\mathcal {AP}}}
\newcommand{\stpf}{{ P_f^0}}
\newcommand{\stpa}{{ P_a^0}}
\def\x{{\mathbf x}}
\def\z{{\mathbf z}}

\def\eq#1{(\ref{#1})}

\newcommand{\RR}{\mathbb{R}}
\newcommand{\PP}{\mathbb{P}}
\newcommand{\lam}{\lambda}
\newcommand{\beq}{\begin{equation}}
\newcommand{\eeq}{\end{equation}}
\newtheorem{theorem}{Theorem}
\newtheorem{proposition}[theorem]{Proposition}
\newtheorem{problem}[theorem]{Problem}
\newtheorem{definition}[theorem]{Definition}
\newtheorem{corollary}[theorem]{Corollary}
\newtheorem{example}[theorem]{Example}
\newtheorem{remark}[theorem]{Remark}
\newtheorem{notation}[theorem]{Notation}
\newtheorem{algorithm}[theorem]{Algorithm}
\newcommand{\pf}{\begin{proof}}
\newcommand{\foorp}{\end{proof}}

\DeclareMathOperator{\tr}{tr}
\begin{document}
\title{Object-image correspondence for curves under finite and affine cameras}
\author{Joseph M. Burdis and Irina A. Kogan}
\address{ Department of Mathematics,
North Carolina State University, Raleigh, NC, 27695, USA\\  (e-mail: {\tt
jmburdis,iakogan@ncsu.edu})}\thanks{This project was partially supported by NSF grant
CCF-0728801 and NSA grant H98230-11-1-0129}
\date{}

\maketitle
\begin{abstract}

  We provide criteria for deciding whether a given planar curve is an image of a given spatial curve, obtained  by a  central  or a parallel projection with unknown parameters.  These criteria reduce the projection problem to a  certain modification of the  equivalence problem of planar curves under affine and projective transformations. The latter problem  can be addressed using Cartan's moving frame method. 
 This leads to a novel algorithmic solution of the projection problem for curves. The computational advantage of the algorithms presented here, in comparison to algorithms based on  a straightforward solution, lies in a significant reduction of a number of real parameters that has to be eliminated in order to establish existence or non-existence of a projection that maps    a given spatial curve to a given planar curve.
 The same approach can be used to decide whether a given finite set of ordered points on a plane is an image of  a given finite set of ordered points in $\RR^3$. The motivation comes from the problem of establishing a correspondence between an object and an image, taken by  a camera with unknown position and parameters.   
 \end{abstract}

\vskip5mm

\section{Introduction}

The problem of identification of objects in 3D with their planar images, taken by a camera with unknown position and parameters, is an important task in computer object  recognition. Since the defining features of many objects can be represented by curves, obtaining an algorithmic  solution for the projection problem for curves is essential, but appears to be unknown in the case of projections with a large number of free  parameters. We address this problem for two classes of cameras: finite projective cameras and affine cameras.

The set of {\em finite projective cameras} (also called {\em finite cameras}) has 11 parameters and corresponds to the set of  {\em central projections.} The set of {\em affine cameras} has 8 parameters and corresponds to the set of  
{\em parallel projections}. An affine camera can be obtained as a limit of a finite camera, as the camera center approaches infinity along the  perpendicular from the camera center to the image plane. See \cite{hartley04} for an overview of camera projections and related geometry. 
An affine camera has fewer parameters and provides a good approximation of a finite camera  when   the distance between a camera and an object is significantly greater than the object depth  \cite{hartley04, stiller06}.  

The projection problem for curves is formulated as follows:
%%%%%%%%%
 \begin{problem}\label{proj-problem} \vskip-1mmGiven a curve  $\oset$ in $\RR^3$ and a  curve  $\iset$ in $\RR^2$, does there exist a finite camera (central projection) or an affine  camera (parallel projection)  that maps 
$\oset$ to $\iset$? 
\end{problem}

A straightforward approach to Problem~\ref{proj-problem}, in the case of central projections, leads to the following real quantifier eliminations problem:
%%%%%
  \begin{problem}\label{proj-problem-v2} \vskip-1mm Given a curve  $\oset$ in $\RR^3$ and a  curve  $\iset$ in $\RR^2$ decide whether   there exist  11 real parameters, which describe a central projection $P$,  such that
$$\forall x\in \iset\,\, \exists z\in \oset, \mbox{ such that }x=P(z)?$$ 
\end{problem}
Real quantifier elimination problems are algorithmically  solvable \cite{tarski:51}.
There is an extensive body of literature devoted to computationally effective methods in real quantifier elimination, including 
\cite{Collins:75}, \cite{Grig88}, \cite{Hong:90a}, \cite{HRS93}, \cite{pedersen93}. High computational complexity of these algorithms  make a reduction in the number of  quantifiers  to be desirable. 

The projection criteria, developed in this paper, reduces the projection problem to the problem of deciding whether
the given planar curve $\iset$ is equivalent to a curve in a certain family of  {\em planar} curves  under an action of the projective group, in  the case of central projections, and under the action of the affine group  in  the case of parallel projections. The family of curves depends  on 3 parameters in   the case of central projections, and on 2 parameters in the case of parallel projections. 

The group-equivalence problem can be solved by an adaptation  Cartan's moving frame method. 
Following this  method for  the case of central projections, when $\oset$ and $\iset$ are rational algebraic curves, we define two corresponding rational signature maps $\sig_\iset\colon\RR\to\RR^2$ and $\sig_\oset\colon\RR^4\to\RR^2$. Construction of these signature  maps requires only differentiation and arithmetic operations and is computationally trivial.
Problem~\ref{proj-problem-v2} reduces to
%%%%
\begin{problem}\label{proj-problem-sig} \vskip-1mm Given two rational maps  $\sig_\iset,\, \RR\to\RR^2$ and $\sig_\oset\colon\RR^4\to\RR^2$ decide whether there exist $c_1, c_2,c_3\in\RR$,  such that
$$\forall t\in \RR,\,\mbox{ where } \sig_\iset(t) \mbox{ is defined},\, \exists s\in \RR, \mbox{ such that }\sig_\iset(t)=\sig_\oset(s, c_1,c_2,c_3).$$ 
\end{problem}
%%%
Thus, the projection criteria developed in this paper allows us to reduce the number of real quantifiers that need to be eliminated from 13 (11 parameters define a central projection, one is needed to parametrize  curve $\oset$ and another one to parametrize curve $\iset$)  to 5. The case of parallel projection is treated in the similar manner and leads to the reduction of the number of real quantifiers that need to be eliminated from 10 to 4.

 Previous works on related problems include \cite{feldmar95}, where a solution to Problem~\ref{proj-problem} is given for \emph{finite cameras with known internal parameters}.  In this case, the number of free parameters is reduced from 11 to 6 parameters, representing  the position and the orientation of the camera. The method presented in \cite{feldmar95} also uses an additional assumption that a planar curve  $\iset\subset\RR^2$ has at least two points, whose tangent lines coincide.
In the current paper we do not assume that the internal camera parameters  are known.

A solution of the projection problem  for finite ordered sets of points under \emph{affine} cameras appeared in \cite{stiller06,stiller07} and served as an inspiration for this paper. In Section~\ref{finite}, we summarize the approach of  \cite{stiller06,stiller07}
and indicate how the solution of Problem~\ref{proj-problem}  may be adapted
to  produce an alternative solution to the projection problem  for finite ordered sets of points under \emph{either affine or finite} cameras.

One of the advantages of the novel approach to the solution of Problem~\ref{proj-problem}, introduced in this paper, 
is its universality: essentially the same method can be adapted to various types of the projections and various types of objects, both continuous and discrete. 
Similar to many  previous considerations of the projection problems,  we utilize actions of  affine and projective groups to obtain  
 a solution to the projection problem. Our literature search did  not yield, however, neither  previous  solutions  for the projection problem for curves, where  cameras with  unknown internal  and external parameters  are considered, nor a similar combination of ideas as presented here. The algorithmic solution presented here,  would have to be fine-tuned to become practically useful in real-life applications, but  we believe, it has a good potential to develop into a practically efficient method. Some directions of such improvement are indicated in Section~\ref{discussion} of the paper. 

Problem~\ref{proj-problem} can be  generalized to higher dimensions as follows:
\begin{problem}\label{proj-problem2} \vskip-1mm Given a  curve  $\oset$ in $\RR^{n}$ and a  curve  $\iset$ in $\RR^{n-1}$ ($n\geq 3 )$ does there exist a central or a parallel  projection from  $\RR^{n+1}$ to a hyperplane in  $\RR^{n}$ that maps $\oset$ to $\iset$? 
\end{problem}
The solution, proposed  in this paper, has a straightforward adaptation to higher dimensions and will result in a reduction in the number of quantifiers that need to be eliminated from $n(n+1)+1$ to $n+2$ for central projections and from $n^2+1$ to $n+1$ for parallel projections.   In this paper, we restrict the presentation to the case  of $n=3$, which has applications  in computer image recognition and presents least computational challenge. 

The paper is structured as follows. After reviewing the  geometry of finite and affine cameras in Section~\ref{cameras}, we define  actions of direct products of affine and projective groups on the set of cameras in Section~\ref{actions}. We use these actions to reduce  Problem~\ref{proj-problem} for finite and affine projections to a modification of  the equivalence problem for  \emph{planar} curves under projective  and affine transformations, respectively.   This leads to the main result of this paper, projection criteria for curves, formulated in Section~\ref{criteria}. In Section~\ref{group-equiv},
 we review a solution for the group-equivalence problem, based on differential signature construction \cite{calabi98}. In Section~\ref{algorithms}, we combine our projection criteria and the differential signature construction in order to obtain  an algorithm for solving the projection  problem and show some examples.  
 Although the projection criteria derived in Section~\ref{criteria} of the paper  are valid for arbitrary classes of curves (and, more generally, for arbitrary subsets   of
 $\RR^3$ and  $\RR^2$, respectively) the computational algorithms of Section~\ref{algorithms} are developed for rational algebraic curves. We will consider a possible generalization  of these algorithms to non-rational algebraic curves in an upcoming paper  \cite{bk-prep}.
In Section~\ref {finite} we indicate how  the approach of this paper can be applied to the projection problem  for finite ordered sets of points. A solution to the latter problem for affine projections appeared in \cite{stiller06,stiller07}. We provide a brief comparison of the two approaches. 
 In Section~\ref{discussion}, we discuss possible variations of our algorithm based on alternative solutions of the group-equivalence problem, as well as possible adaptations  to  curves  presented
 by samples of discrete points whose coordinates may be known only approximately.

%%%%%%%

\section{Finite and affine cameras}\label{cameras}
A simple pinhole camera is shown in Figure~\ref{pinhole-camera} and corresponds to a central projection.
\begin{figure}\centering\epsfig{file=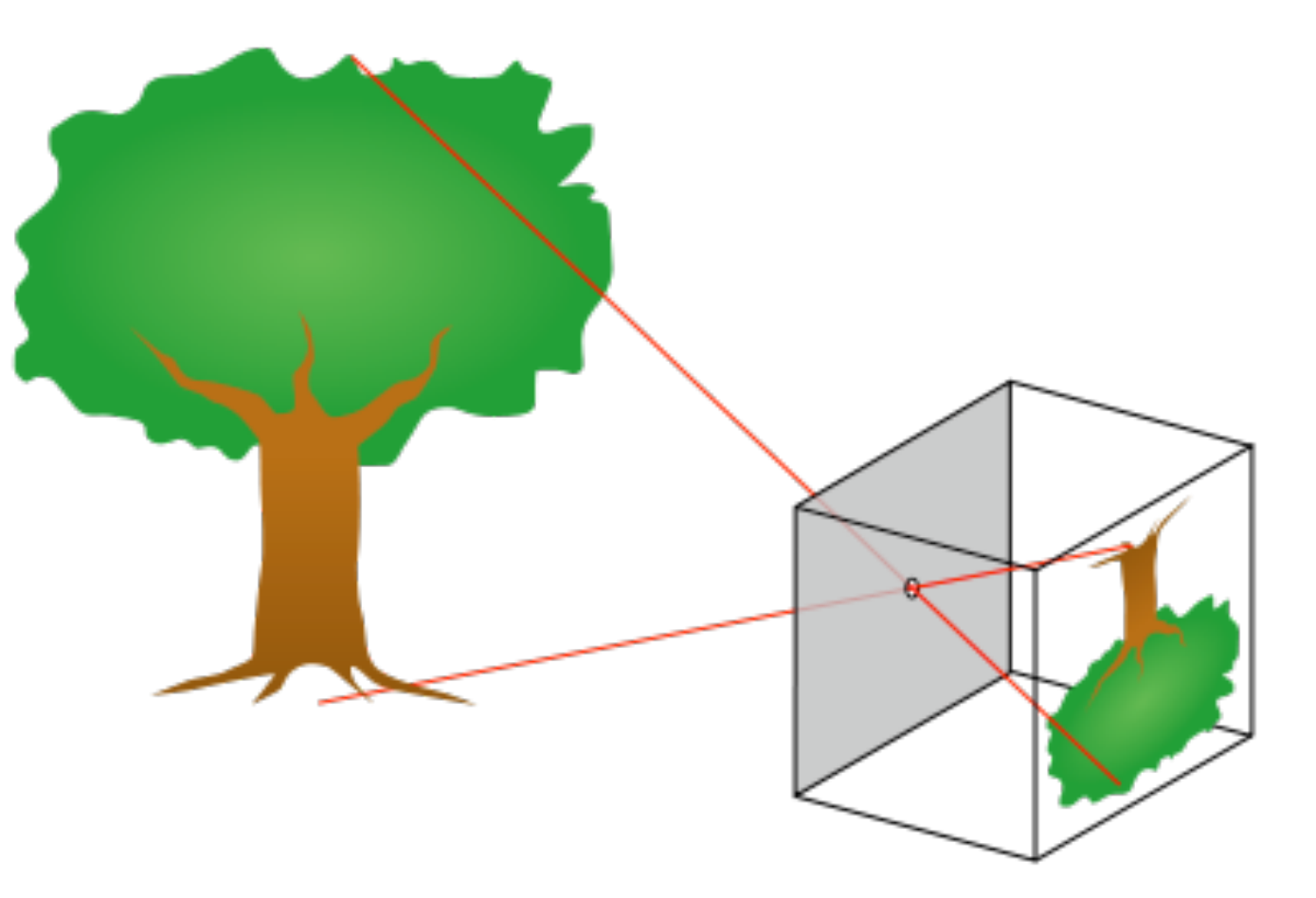,height=1in, width=1.5in}\caption{Pinhole camera \cite{pinhole-figure}.}\label{pinhole-camera}\end{figure}
Let $(z_1,z_2,z_3)$ be   coordinates in $\RR^3$, relative to an orthonormal  coordinate basis, such that   the camera is located at the origin on $\RR^3$ and  the image plane is passing through the point $(0,0,1)$ perpendicular to the $z_3$-axis. We assume that a coordinate system on the image plane is provided by the first two coordinate functions on $\RR^3$, i.e. $x(z_1,z_2,1)=z_1$ and  $y(z_1,z_2,1)=z_2$.
Then a point $(z_1,z_2,z_3)$, such that $z_3\neq 0$ is projected to the point \beq\label{pinhole}(x,y)=\left(\frac {z_1}{ z_3}, \frac {z_2}{ z_3}\right).\eeq
We introduce a freedom  to choose  the position of the camera center ($3$ degrees of freedom), the position of the image plane,  ($3$ degrees of freedom),  as well  
as a choice  of,  in general, non-orthogonal, linear system of coordinates on the image plane ($5$ degrees of freedom, since the overall scale is absorbed by the previous choices, i.e., the choice of the distance between the image plane and the camera center).
For real parameters  $p_{ij}, i=1\dots 3,j=1\dots 4$,  a generic projection maps a point  $(z_1,z_2,z_3)\in\RR^3$  to a point in the image plane with coordinates
\begin{small} 
\begin{eqnarray}\label{proj}
x&=&\frac {p_{11}\, z_1+p_{12}\,z_2+p_{13}\,z_3+p_{14}} {p_{31}\, z_1+p_{32}\,z_2+p_{33}\,z_3+p_{34}},\\
\nonumber  y&=&  \frac {p_{21}\,z_1+p_{22}\,z_2+p_{23}\,z_3+p_{24}} {p_{31}\, z_1+p_{32}\,z_2+p_{33}\,z_3+p_{34}}.
\end{eqnarray}
\end{small}
A convenient matrix representation of this map is obtained by embedding 
$\RR^n$ into projective space $\PP^n$
and utilizing homogeneous coordinates on $\PP^n$. 
%%%%%
\begin{notation} \emph{Square brackets} around matrices (and, in particular,  vectors) will be used 
to denote an equivalence class  with respect to multiplication of a matrix   by a nonzero scalar. Multiplication of equivalence classes of matrices $A$ and $B$
of appropriate sizes is well-defined by $[A]\,[B]:=[A\, B]$.
\end{notation} 
%%%%%%
With this notation 
a point $(x,y)\in\RR^2$ corresponds  to a point $[x,y,1]=[\lam x,\lam y,\lam ]\in\PP^2$ for all  $\lam\neq 0$, and a point  $(z_1,z_2,z_3)\in\RR^3$ corresponds to $[z_1,z_2,z_3, 1]\in\PP^3$. We will refer to the points in $\PP^n$ whose   last homogeneous  coordinate is zero as \emph{points at infinity}. In  homogeneous coordinates   projection \eq{proj} is given by\footnote{superscript $tr$ denotes transposition.}
\beq [x,\,y\, ,1]^{\tr}=[P]\, [z_1,\,z_2,\,z_3,\,1]^{\tr},\eeq
where $P$ is $3\times 4$ matrix of rank 3. 

Matrix $P$ has a $1$-dimensional kernel, i.~e.~there exists a  non-zero point   $( z^0_1,z^0_2,z^0_3,z^0_4)\in\RR^4$ such that  $P\,( z^0_1,z^0_2,z^0_3,z^0_4)^{tr}=(0,0,0)^{tr}$. Therefore, the image of the point $[ z^0_1,z^0_2,z^0_3,z^0_4]\in\PP^3$ under the  projection is undefined (recall that $[0,0,0]$ is not a point in $\PP^2$).  Geometrically the kernel of $P$ corresponds to the center of the projection.  

Camera is called {\it finite} if its center is not at infinity.   In the case of finite cameras  the left $3\times 3$ submatrix of $P$ is non-singular. % Finite cameras correspond to {\em central projections} from $\RR^3$ to a plane.%The set of all such matrices will be called \emph{the set of finite projections} and will be denoted $\fp$.

 On the contrary  an \emph{infinite} camera has its center at an infinite point of $\PP^3$ and so the  left $3\times 3$ submatrix of $P$ is singular. An infinite camera is called \emph{affine} when  the preimage of the line at infinity in $\PP^2$ is the plane at infinity in $\PP^3$.  In this case $[P]$ can be represented by a matrix whose last row is $(0,0,0,1)$. %The set of all such matrices will be called \emph{the set of affine projections} and will be denoted $\ap$. The term  {\em generalized weak perspective projections} is used in \cite{stiller06,stiller07}).
  Affine cameras correspond to {\em parallel projections} from $\RR^3$ to a plane.
   Eight degrees of freedom reflect a choice of the direction of a projection and  a choice  of,  in general non-orthogonal, linear system of coordinates on the image plane. An image plane may be assumed to be perpendicular to the direction of the projection, since other choices are absorbed in the freedom to choose a coordinate system  on the image plane.
   %%%%
\begin{definition}\label{def:cameras}
A set of  equivalence classes $[P]$, where $P=(p_{ij})^{i=1\dots 3}_ {j=1\dots 4}$  is a $3\times 4$ matrix whose  left $3\times 3$ submatrix is non-singular, is called \emph{the set of finite projections} and is denoted $\fp$.  

A set of equivalence  classes $[P]$, where $P=(p_{ij})^{i=1\dots 3}_{ j=1\dots 4}$ has rank 3 and its last row is $(0,0,0,\lambda)$, $\lambda\neq 0$, is called \emph{the set of affine projections} and is denoted $\ap$. Affine projections are also called  {\em generalized weak perspective projections}  \cite{stiller06,stiller07}).
\end{definition}
 %%%%%
Equation \eq{proj} determines a {\em central projection} $\RR^3\to\RR^2$ when $[P]\in \fp$ and it determines a  parallel projection  when $[P]\in \ap$. \footnote{From now on, we  refer to {\em central} projections as {\em finite} projections  and to {\em parallel} projections as {\em affine} projections.}
 
 The sets of finite and affine projections are disjoint. Projections that are not included in these two classes are infinite non-affine projections. These are not frequently used in computer vision and are not considered in this paper.     

A simple pinhole camera projection \eq{pinhole} is represented by the matrix:
\beq\label{stpf} \stpf:=\left[
\begin{array}{cccc}
1 & 0 & 0 & 0 \\
0& 1 & 0 & 0 \\
0 & 0 & 1 & 0 \end{array} \right]\eeq
and  is called the \emph{standard finite projection}.
The {\em standard affine projection} is the orthogonal projection on the $z_1z_2$-plane. It is represented by  the matrix
\beq\label{stpa} \stpa:=\left[
\begin{array}{cccc}
1 & 0 & 0 & 0 \\
0& 1 & 0 & 0 \\
0 & 0 & 0 & 1 \end{array} \right].\eeq
%%%%%%%%%%%%%
%%%%%%%%%%%%%%%%%
\section{Group actions}\label{actions}
Since group actions and, in particular,  actions of the affine and projective groups  play a crucial role in our construction, we review here the relevant definitions:
\begin{definition}\label{act}
An \emph{action} of a group $G$ on a set
  $\set$ is a  map $\act\colon G\times \set \to \set$ that  satisfies the following two
  properties:
\begin{enumerate}
\item $\act(e,s)=s$,  $\forall s\in \set$, where $e$ is the identity of the group.
 \item $\act(g_1, \act(g_2, s))=\act(g_1\, g_2 , s)\),  $\forall\, s \in \set$ and
$\forall\, g_1,g_2\in  G$.
\end{enumerate}
\end{definition}
For  $g\in G$  and  $s\in \set$ we sometimes  write
\(\act(g,s)=g\,s.\)
%%%%%%%%%%%%%%%%
\begin{definition} \label{trans-def} An action is called  \emph{transitive} if for all $s_1,s_2\in \set$ there exists  $g\in G$   such that $s_1=g\,s_2$.\end{definition}
\begin{definition} For a fixed  element $s\in \set$ the set $G_s=\{g\in G|gs =s\}\subset G$ is called  \emph{the stabilizer of $s$}. \end{definition}
It can be shown that a stabilizer $G_s$ is a subgroup of $G$.

%%%%%%%%
\begin{definition}
The \emph{projective group} $\pgl(n+1)$ is a quotient of  the general linear group $\GL(n+1)$, consisting of $ (n+1)\times(n+1)$ non-singular  matrices, by a 1-dimensional abelian  subgroup $\lam I$, where $\lam\neq 0 \in \RR$ and  $I$ is the identity matrix. Elements of $\pgl(n+1)$ are equivalence classes $[B]=[\lam B]$, where $\lam\neq 0$ and    $B\in \GL(n+1)$.
 
The  \emph{affine}  group  $\aff(n)$ is a subgroup   of $\pgl(n+1)$ whose elements $[B]$ have  a representative $B\in \GL(n+1)$ whose last row is $(0,\dots,0,1)$.  

The  \emph{special affine}  group  $\saff(n)$ is a subgroup   of $\aff(n)$ whose elements $[B]$ have  a representative $B\in \GL(n+1)$ with determinant $1$ and the last row equal to $(0,\dots,0,1)$. \end{definition}

In homogeneous coordinates the standard action of the projective group $\pgl(n+1)$  on $\PP^n$  is defined by: 
\beq\label{homog-act}\act([B],[z_1,\dots, z_n, z_0]^{\tr})=[B]\,[z_1,\dots, z_n, z_0]^{\tr}.\eeq
The action \eq{homog-act} induces  almost everywhere defined  linear-fractional action of $\pgl(n+1)$ on $\RR^n$.  In particular, for $n=2$, $[B]\in\pgl(3)$ we have 
\beq\label{pr2} (x,y)\to \left(\frac {b_{11}\, x+b_{12}\,y+b_{13}} {b_{31}\, x+b_{32}\,y+b_{33}}, \frac {b_{21}\, x+b_{22}\,y+b_{23}}{b_{31}\, x+b_{32}\,y+b_{33}}\right).\eeq
The restriction of  \eq{homog-act} to $\aff(n)$  induces an action on $\RR^n$ consisting of   compositions of linear transformations and translations.   In particular, for $n=2$ and  $[B]\in\aff(2)$  represented by a matrix   $B$ whose last row is $(0,0,0,1)$
\beq\label{aff2} (x,y)\to \left(b_{11}\, x+b_{12}\,y+b_{13},\, b_{21}\, x+b_{22}\,y+b_{23}\right).\eeq

 %%%%%%%%%%%%%%%

  \subsection{Action on finite cameras}
  A straightforward exercise in matrix multiplication shows that the map $\act:\left(\pta\right)\times\fp\to\fp$ defined by 
  \beq\label{product-action1}
\act\big(([A],[B]),\,[P])=  [A]\,[P]\,[B^{-1}]
  \eeq
   for  $[P]\in\fp$  and  $([A],[B])\in \pta$,  satisfies Definition~\ref{act} of a  group-action.
 
%%%%%%%%%%%%%%%%%%
\begin{proposition} \label{trans_f}The action of $\pta$ on $\fp$, defined by \eq{product-action1} is transitive.
 \end{proposition}
 %%%%%%%%%%%%%
 \begin{proof}
  According to Definition~\ref{trans-def} we need to prove that for all $[P_1],[P_2]\in \fp$ there exists $([A],[B])\in \pta$ such that $[A][P_1][B^{-1}]=[P_2]$. It is sufficient to prove that for all $[P]\in \fp$ there exists $([A],[B])\in \pta$ such that $[P]=[A]\,[\stpf]\, [B]$, where $[\stpf]$ is the standard finite projection \eq{stpf}.
 A finite projection  is given by a $3\times 4$ matrix  $P=(p_{ij})^{i=1\dots 3}_{j=1\dots 4}$  whose left $3\times 3$ submatrix is non-singular. Therefore there exist $c_1,\,c_2,\,c_3\in\RR$ such that $p_{*4}=c_1\,p_{*1}+c_2\,p_{*2}+c_3\,p_{*3}$, where $p_{*j}$ denotes the $j$-th column of the matrix $P$.
  We define  $A:=(p_{ij})^{i=1\dots 3}_{ j=1\dots 3}$ to be the left $3\times 3$ submatrix of $P$ 
and \beq\label{Bf}
B:=\left(
\begin{array}{cccc}
1 & 0 & 0  &  c_1 \\
0 & 1 & 0  &  c_2 \\
0 & 0 & 1  &  c_3 \\
0 & 0 & 0  & 1 
\end{array} \right).
\eeq
 We observe that $([A],[B])\in\pta$ and  $[A][\stpf] [B]=[P]$.
\end{proof}
%%%%%%%%
\begin{corollary} The set $\fp$ of  finite projections is diffeomorphic to the homogeneous  space $(\pta)\slash H_f^0$, where $H_f^0$ is the 9-dimensional stabilizer of $[\stpf]$.  \end{corollary}
%%%%%%%%%
A straightforward computation shows that
\begin{equation}  \label{Hf}
    H_f^0=\left\{\left([A],\,  
\left[\begin{array}{cc}
 A & \bf{0}^{\tr}\\
\bf{0}& 1
  \end{array} \right] \right)\right\}, \mbox{ where $A\in \GL(3)$.}
  \end{equation}
%%%%%%%%
\begin{remark} It follows from the proof of Proposition~\ref{trans_f}  that any finite projection is a composition of a translation in $\RR^3$ (corresponding to translation of the camera center to the origin), the standard projection \eq{stpf} (pinhole camera), and a projective transformation on the image plane.  \end{remark}

 %%%%%%%%%%%%%
 \subsection{Action on affine cameras}
Formula \eq{product-action1} with  $[P]\in\ap$  and $([A],[B])\in \affp$ 
 defines an action of the direct product $\affp$ on the  set of affine projections $\ap$.

\begin{proposition} \label{transa}The action of $\affp$ on $\ap$, defined by \eq{product-action1}, is transitive.
 \end{proposition}
 \begin{proof}
 It is sufficient to prove that for all $[P]\in \ap$ there exists $([A],[B])\in \affp$ such that $[P]=[A]\,[\stpa]\,[B]$, where $\stpa$ is the standard projection \eq{stpa}.
 An affine  projection $P$ is given by  the matrix  
 \beq\label{aff:proj}
P= \left(\begin{array}{cccc}
  p_{11}&  p_{12}&p_{13}  &p_{14} \\
   p_{21}&  p_{22}&p_{23}  &p_{24}\\
   0& 0  &0 &1
\end{array}\right)\eeq
  of rank 3. Therefore there exist $1\leq i<j\leq 3$ such that the rank of the submatrix
 $
\left(
\begin{array}{cc}
  p_{1i}& p_{1j}   \\
   p_{2i}& p_{2j}   
\end{array}
\right)
$
  is 2.  Then for $1\leq k\leq 3$, such that  $k\neq i$ and $k\neq j$, there exist $c_1,\,c_2\in\RR$, such that 
  $\left(
\begin{array}{c}
  p_{1k}   \\
   p_{2k}   
\end{array}
\right)=c_1\,\left(\begin{array}{c}
  p_{1i}   \\
   p_{2i}   
\end{array}
\right)+c_2\,\left(\begin{array}{c}
  p_{1j}   \\
   p_{2j}   
\end{array}
\right)$.
We define \beq\label{Aa}A:=
\left(\begin{array}{ccc}
  p_{1i}& p_{1j}&p_{14}   \\
   p_{2i}& p_{2j} &p_{24}\\
   0&0&1  
\end{array}
\right)   \eeq
and 
define $B$ to be the matrix whose columns are   vectors $b_{*i}:=(1,0,0,0)^{\tr}$, $b_{*j}:=(0,1,0,0)^{\tr}$, $b_{*k}:=(c_1,c_2,1,0)^{\tr}$, $b_{*4}=(0,0,0,1)^{\tr}$. 
We observe that $([A],[B])\in\affp$ and that $[A][\stpa] [B]=[\stpa]$.
\end{proof}
\begin{remark}\label{B-cases}
Note that there are only three possible values of $(i,j,k)$ in the above proof:
\begin{enumerate}
\item  [if] $(i,\,j,\,k)=(1,2,3)$,  then 
\beq\label{B}B=
\left(\begin{array}{cccc}
  1&  0&c_1  &0 \\
   0& 1& c_2  &0\\
   0& 0 &1  &0\\
   0& 0  &0 &1
\end{array}\right);\eeq
\item [if] $(i,\,j,\,k)=(1,3,2)$, then the corresponding $B$ is obtained by interchanging  the second and the third column in \eq{B};
 \item[if] $(i,\,j,\,k)=(2,3,1)$, then the corresponding $B$ is  obtained by a cyclic shift (by one to the right)  of the first  three columns in \eq{B}.
 \end{enumerate}
\end{remark}
%%%%
\begin{corollary} The set $\ap$ of  affine projections is diffeomorphic to the homogeneous  space $(\affp)\slash H_a^0$, where $H_a^0$ is the 10-dimensional stabilizer of $[\stpa]$.  \end{corollary}
%%%%%
A straightforward computation shows that
\begin{equation}  \label{Ha}
    H^0_a=\left\{ \left( \left[
\begin{array}{cccc}
 m_{11} & m_{12} & a_1 \\
 m_{21} & m_{22} & a_2 \\
 0 & 0 & 1  \end{array} \right], \left[
\begin{array}{cccc}
 m_{11} & m_{12} & 0 & a_1 \\
 m_{21} & m_{22} & 0 & a_2 \\
 m_{31} & m_{32} & m_{33} & a_3  \\
 0 & 0 & 0 & 1  \end{array} \right] \right) \right\},
  \end{equation}

 where $m_{33}\,(m_{11}m_{22}-m_{12}m_{21})\neq 0$.
 
 %%%%%%%%%%%%%%%%
 \section{Projection criteria for  curves}\label{criteria}
 %%%%%%%%%%%%%%%%%%%%%%%
 In this section we formulate criteria for the existence of a finite or an affine projection  that maps a given algebraic curve in $\RR^3$ to a given algebraic curve in $\RR^2$.
 
 We recall \cite{fulton} that for every algebraic curve $\iset\subset\RR^n$ there exists a unique projective algebraic curve   $[\iset]\subset\PP^n$ such that $[\iset]$ is the smallest projective variety containing $\iset$. As before, we  identify a point on $ [\iset]$ with the column vector of its homogeneous coordinates.

 %%%%%%%%%%
 \begin{definition}\label{def-proj} We say that a curve  $\oset \subset \RR^3$  {\em projects onto} $\iset\subset \RR^2$ if there exists a $3\times 4$ matrix $P$ of rank 3  
    such that  the set $[P][\oset]$ is dense in $[\iset]$. In this definition, we allow the center of the projection $[P]$ to lie on $[\oset]$, and if this happens  $[P][\oset]$ is undefined at one point.  
    %\beq\label{proj-eq}C_{\gam}=[P]\,C_{\Gam}:=\left\{\left.[P]\, [z_1(s),\,z_2(s)\,z_2(s),\, 1]^{\tr}\right|s\in I_\Gamma \right\}.\eeq
     \end{definition}
 Note that if      $\oset \subset \RR^3$  {\em projects onto} $\iset\subset \RR^2$ according to Definition~\ref{def-proj} then the image 
 of $\oset$ under the map \eq{proj} is dense in $\iset$. Disregarding possible exclusions of finites sets of points, we write
  $\iset=[P](\oset)$ and $[\iset]=[P][\oset]$ if Definition~\ref{def-proj} is satisfied.

In the next two subsections we show that the projection problem for central and parallel projections can be reduced to  a variation of  the  equivalence problem of planar curves under projective and affine actions, respectively. 
%%%%
\begin{definition}\label{def-equi} We say that two curves  $\iset_1\subset\RR^n$ and  $\iset_2\subset\RR^n$  are $\pgl(n+1)$-equivalent (and also that $[\iset_1],  [\iset_2]\in\PP^n$ are $\pgl(n+1)$-equivalent)   if there exists
$[A]\in\pgl(n+1)$, such that 
\beq \label{eq-prgr}[\iset_2]=\{[A] [x]\,|\, [x]\in[\iset_1]\}.\eeq 
If \eq{eq-prgr} is satisfied for $[A]\in G$, where $G$ is a subgroup of  $\pgl(n+1)$, we say that  $\iset_1$ and  $\iset_2$  are $G$-equivalent.
\end{definition}
%%%%5
We write
  $\iset_2=[A](\iset_1)$ and $[\iset_2]=[A][\iset_1]$ if Definition~\ref{def-equi} is satisfied.
Before stating the projection criteria, we make the following simple, but important  observation.
 
 %%%%%%
\begin{proposition}\label{proj-classes}\begin{enumerate}
 \item [(i)] If $\oset\subset\RR^3$ projects onto $\iset\subset\RR^2$ by an \emph{affine projection}, then any curve that is $\aff(3)$-equivalent to $\oset$ projects onto any curve that is $\aff(2)$-equivalent to $\iset$ by an \emph{affine projection}. In other words, affine projections are defined on affine equivalence classes of curves.
 \item[(ii)] If $\oset\subset\RR^3$ projects onto $\iset\subset\RR^2$ by a \emph{finite  projection} then any curve in $\RR^3$ that is $\aff(3)$-equivalent to $\oset$ projects onto any curve on $\RR^2$ that is $\pgl(3)$-equivalent to $\iset$ by a \emph{finite projection}. 
 \end{enumerate}
\end{proposition}
%%%%%
\begin{proof} (i) Assume that there exists an affine projection $[P]\in \ap$  such that $[\iset]=[P][\oset]$. Then for all $(A,B)\in\affp$ we have  $[A]\,[\iset]=[A]\,[P]\, [B^{-1}]\,\left([B]\,[\oset]\right)$. Since $[A]\,[P]\,[B^{-1}]\in\ap$,  a set  $[B][\oset]$  projects onto $[A] [\iset]$. (ii) is proved similarly.\end{proof}
%%%%
It is \emph{not true} in general that if  a curve $\oset$ can be projected onto two planar curves $\iset_1$ and $\iset_2$ by an affine (or a finite) projection, then  the curves  $\iset_1 $  and  $\iset_2$ are $\aff(2)$-equivalent (or   $\pgl(3)$-equivalent). Counterexamples  appear  in Example~\ref{ex-finite} (for finite projections) and  Example~\ref{ex2} (for affine projections).

%%%%%%
 We are now ready to state and prove  the projection criteria.
 %%%%%%%%
 \begin{theorem} \label{main-finite-camera}({\sc finite projection criteria}) A  curve $\oset\subset\RR^3$  projects onto a  curve $\iset \subset\RR^2$ by a \emph{finite projection}  if and only if there exist  $c_1, c_2, c_3 \in \RR$ such that  the projective  curve 
 \beq\label{poset} [\tilde\oset_{c_1, c_2, c_3}] =\left\{\left[{z_1+c_1},\, {z_2+c_2},\, {z_3+c_3}\right]\,\Big|\, \forall (z_1,z_2,z_3)\in \oset, \right\}\subset\PP^2\eeq
is  $\pgl(3)$-equivalent to $[\iset]\subset\PP^2$.  
\end{theorem}  
\begin{proof}  
($\Rightarrow$)Assume  there exists a finite projection $[P]$  such that  $[\iset]=[P]\, [\oset]$.
 It was established in the proof of Proposition~\ref{trans_f} that $[P]=[A]\,[\stpf]\,[B]$ for some  $[A]\in \pgl(3)$ and  $[B]\in \saff(3)$, where $B$ is given by \eq{Bf}
for some $c_1, c_2, c_3 \in \RR$, and  $\stpf$ is the standard finite projection \eq{stpf}. 
 Therefore  $[\iset] = [A][\stpf]\, [B]\,[\oset]$. Since 
$$[\stpf][B][z_1,z_2,z_3,\,1]^{\tr} = [z_1+c_1,z_2+c_2,z_3+c_3]^{\tr},$$
 then   $[\iset] = [A][\tilde\oset_{c_1,c_2,c_3}]$, where $[\tilde\oset_{c_1,c_2,c_3}]$ is defined by \eq{poset}  
    % Thus   $\iset=[A] \, (\tilde\oset_{c_1,c_2,c_3})$ under the $\pgl(3)$-action \eq{pr2}. 
  %
  
($\Leftarrow$) To prove the converse direction we assume that there exists $[A]\in \pgl(3)$ and $c_1, c_2, c_3 \in \RR$ such that $[\iset]=[A] [\tilde\oset_{c_1,c_2,c_3}]$, where $[\tilde\oset_{c_1,c_2,c_3}]$ is defined by \eq{poset}.
 A direct computation shows that $\oset$ is projected onto $\iset$ by the  finite projection  $[P]=[A]\,[\stpf]\,[B]$, where  $B$ is given by \eq{Bf} and $[\stpf]$ is the standard finite projection \eq{stpf}.
\end{proof}
%%%%%%%%%%%%%
\begin{theorem} \label{main-affine-camera}{(\sc affine projection criteria.)} 
 A curve $\oset\subset\RR^3$  projects onto a curve $\iset\subset\RR^2$  by an affine projection if and only if there exist  $c_1, c_2\in \RR$   and an  ordered triplet $(i,j,k)\in \left\{(1,2,3),\, (1,3,2), \,(2,3,1)\right\}$ such that the {\em planar}    curve
\beq\label{delta-set}\tilde\oset^{i,j,k}_{c_1,c_2} =\left\{\left(z_i+c_1 \,z_k,\,z_j+c_2\,z_k\right)\,\Big|\,  (z_1,z_2,z_3)\in \oset\right\}\subset\RR^2\eeq
is $\aff(2)$-equivalent to  $\iset\subset \RR^2$.  
\end{theorem}
\begin{proof}  
($\Rightarrow$)Assume $\oset$ projects onto $\iset$. Then there exists  an affine projection $[P]\in \ap$ such that  $[\iset]=[P][\oset]$. Recall that the matrix $P$ is of the form \eq{aff:proj} and let
$(i,j,k)$ be a permutation of numbers $(1,2,3)$ such that $i<j$ and
the submatrix of $P$  formed by the $i$-th and $j$-th columns has rank 2. 
 As it was established in the proof of Proposition~\ref{transa} there exist   $[A]\in \aff(2)$  and  $[B]\in \aff(3)$, listed in Remark~\ref{B-cases}, such that $[P]=[A]\,[\stpa]\,[B]$, where $[\stpa]$ is the standard projection \eq{stpa}.
 Since $[\stpa][ B] [\oset]=[\tilde\oset^{i,j,k}_{c_1,c_2}]$, then $[\iset] =[A][\tilde\oset^{i,j,k}_{c_1,c_2}]$ 
 %under the $\aff(2)$-action \eq{aff2} 
 and  the direct statement  is proved. 

($\Leftarrow$) To prove the converse direction we assume that there exist $[A]\in \aff(2)$, two real numbers $c_1$ and $c_2$, and a triplet of indices such that $(i,\,j,\,k)\in \left\{(1,2,3),\, (1,3,2),\,(2,3,1)\right\}$, such that $[\iset]=[A] [\tilde\oset^{i,j,k}_{c_1,c_2}]$, where a planar curve $\tilde \oset^{i,j,k}_{c_1,c_2}(s)$ is given by \eq{delta-set}.
Let  $B$  be a matrix listed in Remark~\ref{B-cases}, corresponding to the $(i,j,k)$-triplet. 
A direct computation shows that $\oset$ is projected onto $\iset$ by the affine projection  $[P]=[A][\stpa][B]$. 
\end{proof}
The families of set  $\tilde\oset^{i,j,k}_{c_1,c_2}$ given by \eq{delta-set}  with $(i,\,j,\,k)\in \left\{(1,2,3),\, (1,3,2),\,(2,3,1)\right\}$ and $c_1,c_2\in\RR$ have a large overlap.  The following corollary eliminates this redundancy and, therefore, is useful for practical computations.
%%%%% Cor
\begin{corollary}\label{reduced-aff-camera}{(\sc reduced affine projection criteria)}  A curve $\oset\subset\RR^3$  projects onto $\iset\subset\RR^2$    by an affine (parallel) projection if and only if  there exist $b,c,f\in \RR$ such that the curve $\iset$ is $\aff(2)$-equivalent to one of the following  \emph{planar} curves 
\begin{eqnarray}\label{toset} \tilde\oset&=&\left\{\left(z_2 ,\,z_3\right)\,\Big|\,  (z_1,z_2,z_3)\in \oset\right\}\subset\RR^2,\\ 
\label{tosetb}\tilde\oset_b&=&\left\{\left(z_1+b \,z_2,\,z_3\right)\,\Big|\,  (z_1,z_2,z_3)\in \oset\right\}\subset\RR^2, \\
\label{tosetcf}\tilde\oset_{c,f}(s)&=&\left\{\left(z_1+c \,z_3,\,z_2+f\,z_3\right)\,\Big|\,  (z_1,z_2,z_3)\in \oset\right\}\subset\RR^2.\end{eqnarray}
 \end{corollary}
\begin{proof}
We first prove that for any permutation $(i,j,k)$ of numbers $(1,2,3)$ such that $i<j$, and for any $c_1,c_2\in \RR$ 
the set $\tilde\oset^{i,j,k}_{c_1,c_2}=\left\{\left(z_i+c_1 \,z_k,\,z_j+c_2\,z_k\right)\,\Big|\,  (z_1,z_2,z_3)\in \oset\right\}$ is $\aff(2)$-equivalent to one of the sets listed in \eq{toset}-\eq{tosetcf}.  

Obviously, $\tilde\oset^{1,2,3}_{c_1,c_2}=\tilde\oset_{c,f}$ with $c=c_1$ and $f=c_2$. 

For $\tilde\oset^{1,3,2}_{c_1,c_2}$, if  $c_2\neq 0$ then $\left(\begin{array}{cc}1&-\frac{c_1}{c_2} \\ 0 &\frac{1}{c_2} \end{array}\right)
\left(\begin{array}{c}z_1+{c_1}z_2\\z_3+c_2{z_2} \end{array}\right)=\left(\begin{array}{c}z_1-\frac{c_1}{c_2}z_3\\z_2+\frac 1 {c_2}{z_3} \end{array}\right)$
and so $\tilde\oset^{1,3,2}_{c_1,c_2}$ is $\aff(2)$-equivalent to $\tilde\oset_{c,f}$ with $c=-\frac{c_1}{c_2}$ and $f=  \frac 1 {c_2}$. Otherwise, if $c_2=0$, the $\tilde\oset^{1,3,2}_{c_1,c_2}=\tilde\oset_b$
with $b=c_1$.

Similarly for $\tilde\oset^{2,3,1}_{c_1,c_2}$,   if $c_2\neq 0$ then $\tilde\oset^{2,3,1}_{c_1,c_2}$ is $\aff(2)$-equivalent to $\tilde\oset_{c,f}$ with $c=\frac{1}{c_2}$ and $f= - \frac {c_1} {c_2}$. Otherwise, if $c_2=0$, then $\tilde\oset^{2,3,1}_{c_1,c_2}(s)=(z_2(s)+c_1z_1(s),\,z_3(s))$. If $c_1\neq 0$ then $\tilde\oset^{2,3,1}_{c_1,c_2}$ is $\aff(2)$-equivalent to $\tilde\oset_b$ with $b=\frac 1{c_1}$, otherwise  $c_1=0$ and $\tilde\oset^{2,3,1}_{c_1,c_2}=\tilde\oset$.

We can  reverse the argument and show that any curve given by \eq{toset}-\eq{tosetcf} is $\aff(2)$-equivalent to a curve from  family \eq{delta-set}.
Then the reduced criteria follows from Theorem~\ref{main-affine-camera}. 
  \end{proof}

%%%%%%%%%%%%%%%%
\section{Group-equivalence problem}\label{group-equiv}
%%%%%%%%%%%%%%%%%%%%%%
Theorems~\ref{main-finite-camera} and~\ref{main-affine-camera} reduce the projection  problem to the problem of establishing group-action equivalence between a given curve and a curve from a certain family. A variety of methods exist to solve group-equivalence problem for curves. We base our algorithm on the differential signature construction described in \cite{calabi98} which originates from   Cartan's moving frame method  \cite{C37}. We consider the possibility  of using some other methods in Section~\ref{discussion}. 
%%%%%%%%%%
\subsection{Differential invariants for planar curves}  In this section we consider rational algebraic  curves, i.~e.~ curves $\gamma (t)=(x(t),y(t))$ defined by a rational map $\gamma\colon \RR\to\RR^2$, defined on $\RR$, with a possible exclusion of  a finite set of points, where the denominators of $x(t)$ or $y(t)$ are zero.\footnote{Throughout the paper, when we make  a statement about a rational map, we assume, without saying so,  that the statement  holds on the domain of the definition of the map.}  By $C_\gamma$ we denote the image of $\gamma$ in $\RR^2$.  The case of non-rational algebraic curves will be considered in \cite{bk-prep}.

An action of a group $G$ on  $\RR^2$ induces an action on curves in $\RR^2$. Using the chain rule, this action can be \emph{prolonged} to the $k$-th order jet space of curves denoted by  $\jet^k$. Variables  $x,y,  \dot x, \dot y, \ddot x,\ddot y,\dots$, which represent  the derivatives of $x,y$ with respect to the parameter of orders from $0$ to $k$, serve as coordinate functions on $\jet^k$.
%%%%%%%%%%%%%%%
\begin{definition} \emph{Restriction of  a function} $\f$
 on $\jet^k$ to  a curve $
 \gamma(t)=(x(t),y(t))$, $t\in \RR$, is a single-variable function  $\f|_{\gamma}(t):=\f\left(x(t),y(t),\frac{dx(t)}{dt},\frac{dy(t)}{dt}, \frac{d^2x(t)}{dt^2}, \dots\right)$. 
 
 A function $\f$ on $\jet^k$ is \emph{invariant under reparameterizations} if for all rational curves $\gamma\colon\RR\to\RR^2$ and for all rational maps $\phi\colon\RR\to\RR$, we have $\f|_\gamma(\phi(t))= \f|_{\tilde\gamma}(t)$, where  $\tilde\gamma(t)=\gamma(\phi(t))$.
  \end{definition}
  %%%%%%%%
For example, $\dot x|_{\tilde\gamma}(t)=\dot x|_\gamma(\phi(t))\phi'(t)$, and hence $\dot x$ is not invariant under reparameterizations, but $\frac{\dot x}{\dot y}$ is invariant under reparameterizations. 
\begin{definition} 
 A \emph{$k$-th order differential invariant} is a function on $\jet^k$ that depends on $k$-th order jet variables and  is  invariant 
 under the prolonged action of $G$ and reparameterizations of curves.
 \end{definition}
 %%%%%%%%%%%
 For example, for  the action of the $3$-dimensional  Euclidean group, consisting of rotations, translations and  reflections  on the plane,   the curvature 
 $\kappa=\frac{\ddot y\dot x-\ddot x\dot y}{\sqrt{\dot x^2+\dot y ^2}}$ is (up to a sign) a lowest order differential  invariant. 
 The sign of $\kappa$ changes when a curve is reflected,  rotated by $\pi$ radians or traced in the opposite direction 
($\kappa^2$ is invariant under the full Euclidean group).  
Higher order differential invariants are obtained by differentiation of curvature with respect to Euclidean  arclength $ds={\sqrt{\dot x^2+\dot y ^2}}\,dt$, i.~e.~$\kappa_s=\frac{d\,\kappa}{d\,s}=\frac 1{\sqrt{\dot x^2+\dot y ^2}}\frac {d\,\kappa}{d\,t}$. Any other Euclidean differential invariant can be locally expressed as a function of $\kappa, \kappa_s,\kappa_{ss},\dots$.

For the majority of Lie group actions on $\RR^2$, a lowest order differential invariant appears at order $r-1$ where $r=\dim G$.
 % and the invariant differential form appears at order $r-2$. 
 The group actions on the plane  with this property are called \emph{ordinary}. 
 All actions considered in this paper are ordinary.  A lowest order differential  invariant for an ordinary action of a group $G$ is called \emph{$G$-curvature}, and a lowest order invariant  differential form is called infinitesimal $G$-arclength. Any differential invariant with respect to the $G$-action can be locally expressed as a functions of $G$-curvature and its derivatives with respect to $G$-arclength. Affine and projective curvatures and infinitesimal arclengths are well known, and can be expressed in terms of Euclidean invariants \cite{faugeras94, kogan03}.
 
 In particular, $\saff$-curvature
 $\mu$ and infinitesimal $\saff$-arclength $d\alpha$ are expressed in terms of their Euclidean counterparts  as follows: 
\beq\label{affc}\mu=\frac{3\,\kappa\,(\kappa_{ss}+3\,\kappa^3)-5\,\kappa_s^2}{9\,\kappa^{8/3}},\quad d\alpha=\kappa^{1/3}ds.\eeq
 $\saff(2)$-curvature has the differential order $4$.  Any $\saff$-differential invariant can be locally expressed as a function of $\mu$ and its derivatives with respect to the $\saff$-arclength: $\mu_\alpha=\frac {d\mu}{d\alpha},\, \mu_{\alpha\alpha}= \frac {d\mu_\alpha}{d\alpha},\,\dots$.
 $\saff$-curvature is undefined for straight lines ($\kappa=0$) and $\frac d{d\alpha}$ is undefined at the inflection points of a curve. 
It is shown, for instance, in \cite{Gug63} that $\mu|_\gamma$  is constant if and only if $C_\gamma$ is a conic. Moreover, $\mu|_\gamma\equiv 0$ if and only if $C_\gamma$ is a parabola,  $\mu|_\gamma$ is a positive constant 
 if and only if $C_\gamma$ is  an ellipse, and  $\mu|_\gamma$ is a negative constant if and only if $C_\gamma$ is  a hyperbola.
  
  By considering the  effects of scalings and reflections   on $\saff(2)$-invariants, we obtain two lowest order $\aff(2)$-invariants that are \emph{rational functions} in jet variables:
\beq\label{aff-inv}J_a=\frac{(\mu_\alpha)^2}{\mu^3},\quad K_a= \frac{\mu_{\alpha\alpha}}{3\,\mu^2}.\eeq

$\pgl(3)$-curvature  $\eta$ and infinitesimal arclength $d\rho$ are expressed in terms of their $\saff$-counterparts: 
\beq\label{pc}\eta=\frac{6\mu_{\alpha\alpha\alpha}\mu_\alpha-7\,\mu_{\alpha\alpha}^2-9\mu_\alpha^2\,\mu}{6\mu_\alpha^{8/3}},\quad d\rho=\mu_\alpha^{1/3}d\alpha.\eeq
The two lowest order \emph{rational} $\pgl(3)$-invariants
\beq\label{proj-inv}J_p=\eta^3,\quad K_p= \eta_\rho.\eeq
are of differential order 7 and 8, respectively.

 \begin{definition}\label{G-excep} A curve $\gamma$ is called  {\em $\aff(2)$-exceptional} if invariants \eq{aff-inv} are undefined on a one-dimensional subset of $C_\gamma$. Equivalently
 $C_\gamma$ is   a straight line or   a parabola. In the former case its Euclidean curvature $\kappa|_\gamma\equiv0$, while in the latter case  its $\saff$-curvature $\mu|_\gamma\equiv0$.
 
 A curve $\gamma$ is called  {\em$\pgl(3)$-exceptional} if invariants \eq{proj-inv} are undefined on a one-dimensional subset of $C_\gamma$. Equivalently, 
 $C_\gamma$ is a straight line or  a conic. In the latter  case  $\mu|_\gamma$ is a  constant.
 \end{definition}

\subsection{Differential signature for planar curves}
  
Following \cite{calabi98} we will use differential signatures to solve the equivalence problem for curves under a group action.
%%%%%
\begin{definition}
Let $J_G$ and $K_G$ be differential invariants of  orders $r-1$ and $r$, respectively,  for an ordinary action of an $r$-dimensional Lie group $G$ on the plane. A \emph{$G$-signature} of a non-exceptional parametric curve $\gamma(t)=(x(t),y(t))$, $t\in\RR$, is a parametric curve $\sig_\gamma(t)=\big(J_G|_\gamma(t), {K_G}|_\gamma(t)\big)$.

 \end{definition}
We note that a signature $S_\gamma\colon \RR\to\RR^2$ of a rational curve $\gamma(t)$, which is defined using rational $G$-invariants, such as given by \eq{aff-inv} or \eq{proj-inv}, is again a rational curve $\RR\to\RR^2$.  In a {\em degenerate} case 
the image of $S_\gamma$ consists of a single point in $\RR^2$: 
$$\exists (j,\, k)\in \RR^2, \mbox{ such that } J_G|_\gamma(t)\equiv j, \quad {K_G}|_\gamma(t)\equiv k,\quad \forall t\in \RR.$$
Curves with  degenerate  signatures are symmetric with respect to a  one-dimesional subgroup of  $G$. For example,  circles  and lines  have constant Euclidean signatures. A circle is symmetric under rotations about its center and a line is symmetric under translations along itself.
 
  It follows from the definition of invariants that the image $\sigs_{\gamma}:=\{\sig_\gamma(t)|t\in \RR\}$ is invariant under reparametrizations  of  the curve $\gamma$
and that the following theorem holds:
%%%%%%%
\begin{theorem} \label{eq-sig}If two non $G$-exceptional planar rational  curves  $C_\alpha$ and $C_\beta$ are $G$-equivalent then  the images of their $G$-signatures coincide: $\sigs_{\alpha}=\sigs_{\beta}$.
\end{theorem}
%%%%%
 Theorem~\ref{eq-sig} is valid not only for rational curves, but for all classes of curves to which the definition of signature can be reasonably adapted, and, in particular, for curves with arbitrary smooth parameterizations.  Examples in \cite{musso09} suggest that one has to be careful when stating the converse
 of this theorem for arbitrary smooth curves. {Theorem~8.53} of \cite{olver::inv}  shows that  the converse is true for curves $y=f(x)$ where $f\colon\RR\to\RR$ is an analytic function. In \cite{bk-prep} we show that this proof can be adapted to the case of rational algebraic curves and obtain the following result:

 %%%%%%%%% 
 \begin{theorem} \label{sig-eq-strong} Two non $G$-exceptional planar rational  curves  $C_\alpha$ and $C_\beta$ are $G$-equivalent if and only if their $G$-signatures coincide: $\sigs_{\alpha}=\sigs_{\beta}$.  
 
 A $G$-exceptional curve is not $G$-equivalent to any of non $G$-exceptional curves.
\end{theorem}
\begin{remark}\label{implicit}
 Signature construction reduces the problem of $G$-equivalence of rational algebraic  curves to the problem of deciding whether two  rational maps from $\RR$ to $\RR^2$ (that represent the signatures of the given curves)  have the same images. The implicit equation $\impsig_\gamma(K,J)=0$ for the signature curve can be computed by an elimination algorithm as outlined, for instance, in Section 3.3 of \cite{CLO96}. When comparing signatures using their implicit equations, one has to be aware that, since   $\RR$ is not an algebraically closed field, two non overlapping signature curves can have the same implicit equation as shown by Example~8.69 in \cite{olver::inv}. \end{remark}
%%%%%%%%%
\section{Algorithms and Examples}\label{algorithms}

In this section, we outline the algorithms for solving projection problems  based on a combination of  the projection criteria of Section~\ref{criteria} and  the group equivalence criterion of Section~\ref{group-equiv}. The detailed algorithms, which  also cover $G$-exceptional curves, and  their preliminary {\sc Maple} implementation are posted on   \url{www.math.ncsu.edu/~iakogan/symbolic/projections.html}. We illustrate the algorithms by several examples. Additional examples can be found  at the above link.

%%%%%%%%%%%%%%%%%%%
\subsection{Finite projections.}
The following algorithm is based on the finite projection criteria stated in Theorem~\ref{main-finite-camera}. 
%%%%
\begin{algorithm}\label{alg-finite}{\sc (Outline for finite projections.)}
\vskip2mm
\noindent INPUT:  a planar  curve $\gamma(t)=\left(x(t),\,y(t)\right)$, $t\in\RR$, and a 
 spatial curve  $\Gamma(t)=\big(z_1(s),z_2(s),z_3(s)\big),$ $s\in\RR$, with rational parameterizations. 
\vskip2mm

\noindent OUTPUT:  YES or NO answer to the question "Does there exist a finite projection $[P]$, such that $[C_\gamma]=[P][C_\Gamma]$  is satisfied?".
\vskip2mm
\noindent STEPS: 
\begin{enumerate}
\item  if $\gamma$ is $\pgl(3)$-exceptional (a straight line or a conic) then follow a special procedure,  else
\item evaluate  $\pgl(3)$-invariants given by \eq{proj-inv} on $\gamma(t)$. The result  consists  of two rational functions  $J_p|_\gamma(t)$ and $K_p|_\gamma(t)$ of $t$;
% defining a rational map $\sig_\gamma=\big(J_p|_\gamma(t), {K_p}|_\gamma(t)\big)\colon\RR\to\RR^2$;
%
\item for arbitrary  $c_1,c_2,c_3\in\RR$ define a curve   $\epsilon_{c_1,c_2,c_3}(s)=\left(  \frac {z_1(s)+c_1} {z_3(s)+c_3},\, \frac {z_2(s)+c_2} {z_3(s)+c_3} \right)$;

\item evaluate   $\pgl(3)$-invariants given by \eq{proj-inv} on $\epsilon_{c_1,c_2,c_3}(s)$ -- obtain two rational functions $J_p|_\epsilon(c_1,c_2,c_3, s)$ and  $K_p|_\epsilon(c_1,c_2,c_3,s)$ of  $c_1,c_2,c_3$ and $s$;

%%%%%%%
\item    if $\exists c_1,c_2,c_3\in \RR$, such that $\forall t\in\RR$, where denominators of   $J_p|_\gamma(t)$ and $K_p|_\gamma(t)$ are non-zero, $\exists s\in \RR$:
\beq\nonumber J_p|_\epsilon(c_1,c_2,c_3,\,s)=J_p|_\gamma(t)\mbox{ and }  K_p|_\epsilon(c_1,c_2,c_3,\,s)=K_p|_\gamma(t),\eeq
then OUTPUT: { \sc YES}, else OUTPUT: { \sc NO}.

\end{enumerate}\end{algorithm}
%%%%%%%

If the output is YES then, in many cases, we can, in addition to establishing the existence of
$c_1,c_2,c_3$ in Step 8 of the algorithm,  find at least one of such triplets explicitly. We then know that  $C_\Gamma$ can be projected to $C_\gamma$ by a projection  centered at $(-c_1,-c_2,-c_3)$. 

We can also, in many cases, determine explicitly a transformation $[A]\in \pgl(3)$ that maps $C_\gamma$ to $C_{\epsilon_{c_1,c_2,c_3}}$. We then know that  $C_\Gamma$ can be projected to $C_\gamma$ by the projection  $[P]=[A][\stpf][B]$, where 
$\stpf$ is defined by \eq{stpf} and $B$ is defined by \eq{Bf}.
%%%%%
 \begin{example}\label{ex-finite}
We would like to decide if the spatial curve 
\begin{equation} \Gamma(s)=\left(z_{1}(s),\,z_{2}(s),
z_{3}(s)  \right) = 
\left( s^3,\,
s^2,\,s  \right),\,s\in\RR
\end{equation}
projects onto any of three given planar curves for $t\in\RR$:
\begin{eqnarray*} \gamma_{1}(t)&=&\left(t^2\, , \, t \right),\, \\
\gamma_{2}(t)&=&\left(\frac{t^3}{t+1} \, , \, \frac {t^2} {t+1} \right), \, \\
\gamma_{3}(t)&=&\left(t, \, t^5\right).
\end{eqnarray*}
For $c_1,c_2,c_3\in\RR$ we define a curve
 \beq\label{eps-ex}\epsilon_{c_1,c_2,c_3}(s)=\left(  \frac {s^3+c_1} {s+c_3},\, \frac {s^2+c_2} {s+c_3} \right).\eeq
 
Since parabola  $\gamma_1(t)$ is $\pgl(3)$-exceptional, its $\pgl(3)$-signature is undefined. It is known that all planar conics are  $\pgl(3)$-equivalent and so, from Theorem~\ref{main-finite-camera}, we know that  $C_\Gamma$ can be projected to $C_{\gamma_1}$  if there exist $c_1,c_2,c_3\in\RR$, such that
the curve defined by \eq{eps-ex} is a conic. This is obviously true for $c_1=c_2=c_3=0$.  Indeed, on can check that $C_\Gamma$ can be projected to $C_{\gamma_1}$ by the standard finite projection \eq{stpf}.

The curve   $\gamma_2(t)$ is not $\pgl(3)$-exceptional, but has a degenerate signature:
$$ J_p|_{\gamma_2}(t)\equiv\frac{250047}{12800}\mbox{ and } J_p|_{\gamma_2}(t)\equiv0, \quad \forall t\in\RR.$$
Following Algorithm~\ref{alg-finite}, we need to decide  whether there exist $c_1,c_2,c_3\in\RR$, such that the restriction of invariants \eq{proj-inv} to the curve  defined by \eq{eps-ex} have the same values $J_p|_{\epsilon}(s)=\frac{250047}{12800}\mbox{ and } J_p|_{\epsilon}(s)=0,\,\forall s\in\RR$. This is, indeed, true for  $c_1=c_2=0$  and $c_3=1$. We can check that $C_\Gamma$ can be projected to $C_{\gamma_2}$ by the a finite projection
\beq\nonumber  P:=\left[
\begin{array}{cccc}
1 & 0 & 0 & 0 \\
0& 1 & 0 & 0 \\
0 & 0 & 1 & 1 \end{array} \right].\eeq
\end{example}
It is important to observe that, although $C_\Gamma$ can be projected to both  $C_{\gamma_1}$  and  $C_{\gamma_2}$, the last two curves are {\em not}
  $\pgl(3)$-equivalent. This underscores an observation made after Proposition~\ref{proj-classes}. 
  
  The  curve   $\gamma_3(t)$ also  has a degenerate signature:
$$ J_p|_{\gamma_3}(t)\equiv\frac{1029}{128}\mbox{ and } J_p|_{\gamma_3}(t)\equiv 0.$$
Following Algorithm~\ref{alg-finite}, we need to decide  whether there exist $c_1,c_2,c_3\in\RR$ such that restriction of invariants \eq{proj-inv} to the curve  defined by \eq{eps-ex} have the same values $J_p|_{\epsilon}(s)=\frac{1029}{128}\mbox{ and } J_p|_{\epsilon}(s)=0,\,\forall s\in\RR$. Substitution of several values of $s$, yields a system of polynomial equations for  $c_1,c_2,c_3\in\RR$ that has no solutions. We conclude that there is no finite projection  from $C_\Gamma$ to $C_{\gamma_3}$.

%%%%%%%%%%%%

 \subsection{Affine  projections}

The following algorithm is based on the reduced affine projection criteria stated in Corollary~\ref{reduced-aff-camera}.  
\begin{algorithm}\label{alg-affine}{\sc (Outline for affine projections.)}
\vskip2mm
\noindent INPUT:  a planar  curve $\gamma(t)=\left(x(t),\,y(t)\right)$, $t\in\RR$, and a 
 spatial curve  $\Gamma(t)=\big(z_1(s),z_2(s),z_3(s)\big),$ $s\in\RR$, with rational parameterizations. 
\vskip2mm
\noindent OUTPUT:  YES or NO answer to the question "Does there exist an affine projection $[P]$, such that $[C_\gamma]=[P][C_\Gamma]$  is satisfied?".
\vskip2mm
\noindent STEPS: 
\begin{enumerate}
\item  if $\gamma$ is $\aff(2)$-exceptional (a straight line or a parabola) then follow a special procedure,  else
\item evaluate  $\aff(2)$-invariants  given by \eq{aff-inv} on $\gamma(t)$. 
 The result  consists  of two rational functions $J_a|_\gamma(t)$ and $K_a|_\gamma(t)$ of $t$;
 \item define a curve $\alpha(s)=(z_2(s), z_3(s))$;
 \item  evaluate   $\aff(2)$-invariants given by \eq{aff-inv} on $\alpha(s)$ -- obtain two rational functions $J_a|_\alpha(s)$ and  $K_a|_\alpha(s)$ of   $s$;
\item  if $\forall t\in\RR$, where denominators of   $J_a|_\gamma(t)$ and $K_a|_\gamma(t)$ are non-zero, $\exists s\in \RR$:
\beq\nonumber J_a|_\alpha(s)=J_a|_\gamma(t)\mbox{ and }  K_a|_\alpha(s)=K_a|_\gamma(t),\eeq
then OUTPUT: { \sc YES} and exit the procedure, else
\item for arbitrary  $b\in\RR$ define a curve   $\beta_{b}(s)=\left(  {z_1(s)+b\,z_2(s)} ,\, z_3(s) \right)$;
\item evaluate   $\aff(2)$-invariants given by \eq{aff-inv} on $\beta_{b}(s)$ -- obtain two rational functions $J_a|_\beta(b, s)$ and  $K_a|_\beta(b,s)$ of  $b$ and $s$;
\item  if $\exists$ $b\in\RR$, such that $\forall t\in\RR$, where denominators of   $J_a|_\gamma(t)$ and $K_a|_\gamma(t)$ are non-zero, $\exists s\in \RR$:
\beq \nonumber J_a|_\beta(b, s)=J_a|_\gamma(t)\mbox{ and }  K_a|_\beta(b,s)=K_a|_\gamma(t),\eeq
then OUTPUT: { \sc YES} and exit the procedure, else
 \item for arbitrary  $c,\,f \in\RR$ define a curve   $\delta_{c,f}(s)=\left(   {z_1(s)+c\,z_3(s)} ,\,z_2+f\, z_3(s) \right)$;
\item evaluate   $\aff(2)$-invariants given by \eq{aff-inv} on $\delta_{c,f}(s)$ -- obtain two rational functions $J_a|_\delta(c,f, s)$ and  $K_a|_\delta(c,f,s)$ of  $c,f$ and $s$;
\item  if $\exists$ $c,f\in\RR$ , such that $\forall t\in\RR$, where denominators of   $J_a|_\gamma(t)$ and $K_a|_\gamma(t)$ are non-zero, $\exists s\in \RR$:
\beq \nonumber J_a|_\delta(c,f, s)=J_a|_\gamma(t)\mbox{ and }  K_a|_\delta(c,f,s)=K_a|_\gamma(t),\eeq
then OUTPUT: { \sc YES} else OUTPUT: { \sc NO}.
\end{enumerate}\end{algorithm}
%%%%%%%%%%
Although the algorithm for affine projections includes more steps then its finite projection counterpart, it is computationally less challenging. If the output is YES then, in many cases, we can find an affine projection explicitly.
%%%%%%%
\begin{example}\label{ex1}
In order to decide whether the spatial curve
\begin{equation*} \Gamma(s)=\left(z_{1}(s),\,z_{2}(s),\,z_{3}(s) \right) = 
\left( s^4+1,\, s^2, \,s \right),\,s\in\RR,
\end{equation*}
can be projected onto
$ \gamma(t)=\left( t \, , \, t^4+t^2 \right),\,t\in\RR
$
by an affine projection, we start by determining that  $\gamma$ is not an $\aff(2)$-exceptional curve (neither a straight line or a parabola). The curve $\gamma$ has non-constant $\aff(2)$-invariants \eq{aff-inv} that satisfy  the following implicit signature equation:
\beq \label{s1} -448\,J^2+(3780\,K+14525)\,J+245\,K^3+40000 -6000\,K-1575\,K^2 = 0
\eeq
Following Algorithm~\ref{alg-affine}, we first check whether $\gamma(t)$ is $\aff(2)$-equivalent to 
$\alpha(s)=\left(z_2(s),\,z_3(s)\right)=\left( s^2, \,s \right)$. The answer is no, since $\alpha(s)$ is an $\aff(2)$-exceptional curve (parabola)  and $\gamma(t)$ is not $\aff(2)$-exceptional. We next check   whether there exists $b\in\RR$ such that $\gamma(t)$ is $\aff(2)$-equivalent to 
$\beta_b(s)=\big(\,z_1(s)+b\,z_2(s),\,z_3(s)\big)=\big( s^4+1+b\,s^2,\,s\big)$.  
We evaluate invariants \eq{aff-inv} on $\beta_b(s)$: 
\begin{eqnarray}\label{j1} J_a|_{\beta_b}(s)=\frac{100\,s^2\,(3\,b-14s^2 )^2}{(b-14s^2 )^3},
\\ \label{k1}  K_a|_{\beta_b}(s)=\frac{-5\,(140\,s^4 -56\,b\,s^2+b^2)}{(b-14\,s^2)^2}.\end{eqnarray}
When $b=0$ the invariants are constant: $J_a|_{\beta_0}(s)\equiv -50/7$ and $K_a|_{\beta_0}(s)\equiv -25/7$,
and, therefore, ${\beta_0}(s)$ is not $\aff(2)$-equivalent to $\gamma(t)$. For all $b\neq 0$ the invariants
\eq{j1} and \eq{k1} are non-constant   and satisfy the signature equation \eq{s1}.

This provides a necessary condition and a strong  indication that   $\gamma(t)$ is $\aff(2)$-equivalent to $\beta_{b}(s)$ for $b\neq 0$. For $b=1$ this  $\aff(2)$-equivalence is obvious, and  hence  $\Gamma(s)$ projects onto $\gamma(t)$ by an affine projection.
\end{example}

%%%%%%%
\begin{example}\label{ex2}
We would like to decide if the spatial curve 
\begin{equation} \Gamma(s)=\left(z_{1}(s),\,z_{2}(s),
z_{3}(s)  \right) = 
\left( s^2+s,\,
s^3-3\,s^2 ,\,
s^4  \right),\,s\in\RR
\end{equation}
projects onto any of three given planar curves for $t\in\RR$:
\begin{eqnarray*} \gamma_{1}(t)&=&\left(t^4+t\, , \, t^2 \right),\, \\
\gamma_{2}(t)&=&\left(t^3-t \, , \, t^3+t^2 \right), \ \\
\gamma_{3}(t)&=&\left(t/{(1+t^3)}, \, t^2/{(1+t^3)}\right) \mbox{(Folium of Descartes).}
\end{eqnarray*}
None of given $\gamma$'s is $\aff(2)$-exceptional and the implicit  equations of their $\aff(2)$-signatures are given, respectively, by: 
\begin{small}  
\begin{eqnarray}\label{sig1} &\hat S_{\gamma_1}:& (165+75\,K)\,J-448-560\,K-175\,K^2 = 0,\\
\label{sig2} & \hat S_{\gamma_2}:& 9261\,J^2-(26460\,K+132300)\,J+160\,K^3+160000+264000\,K+12900\,K^2=0,\\
\label{sig3}&\hat S_{\gamma_3}:& 10\, K+1 =0.
\end{eqnarray}
\end{small}  
Following Algorithm~\ref{alg-affine}, we establish
that $\alpha(s)=\big(z_2(s),z_3(s)\big)$ and  $\beta_b(s)=\big(\,z_1(s)+b\,z_2(s),\,z_3(s)\big)$, for all $b\in\RR$, are not $\aff(2)$-equivalent to either of $\gamma$'s. 

We then establish   that $\delta_{c,f}(s)=\big(z_1(s)+c \,z_3(s),\,z_2(s)+f\,z_3(s)\big)$ is $\aff(2)$-equivalent to
$\gamma_{1}$ when $c=0$ and $f=1/2$ and is $\aff(2)$-equivalent to
$\gamma_{2}$ when $c=0$ and $f=0$, but there are no real values of $f$ and $c$ such that  $\delta_{c,f}(s)$ and $\gamma_3$ are 
 $\aff(2)$-equivalent.

We conclude that there are affine projections of  $\Gamma(s)$ onto both $\gamma_1(t)$ and $\gamma_2(t)$, but not onto $\gamma_3(t)$.  

We note that, although $\Gamma(s)$ affinely projects to both $\gamma_1(t)$ and $\gamma_2(t)$, the curves   $\gamma_1(t)$ and $\gamma_2(t)$ are not $\aff(2)$-equivalent because their signatures have different implicit equations   \eq{sig1} and \eq{sig2}.
This illustrates that the converse to Proposition~\ref{proj-classes} does not hold.

\end{example}

%%%%%%%%%%%%%
\section{Projection  of finite ordered sets (lists) of points}\label{finite}
%%%%%%%%%%%%%%%%%%%%%%%%%
In  \cite{stiller06, stiller07}, the authors present a solution to the problem  of deciding whether or not there exists an affine projection of a list $Z=(\z^1,\dots, \z^m)$ of $m$ points 
 in $\RR^3$ to a list  $X=(\x^1,\dots, \x^m)$  of  $m$ points in $\RR^2$, without finding a projection explicitly. 
 They identify  the lists  $Z$ and $X$ with the elements  of certain Grassmanian   spaces and use  Pl\"uker embedding  of Grassmanians into projective spaces  to explicitly define the  algebraic variety that characterizes pairs of sets related by an affine projection.
 
We indicate here  how our approach  leads to an alternative solution for the projection problem for lists of  points.  Details  of this adaptation  appear in the dissertation \cite{burdis10} and in an upcoming paper \cite{bk-prep}. 
%
%%%%%%%%%
 \begin{theorem} \label{main-finite-points} {(\sc finite projection criteria for lists of points.)} 
  A given list $Z=(\z^1,\dots, \z^m)$ of $m$ points 
 in $\RR^3$ with coordinates $\z^i=(z^r_1,z^r_2,z^r_3)$, $r=1\dots m$ projects onto a given list $X=(\x^1,\dots, \x^m)$  of  $m$ points in $\RR^2$  with coordinates $\x^r=(x^r,y^r)$ by a finite  projection
  if and only if there exist  $c_1, c_2, c_3 \in \RR$  and $[A]\in\pgl(3)$, such that 
  \beq \label{points-finite} [x^r,y^r,1]^{tr}=[A] [ z_1^r+c_1, \, z_2^r+c_2,\, z_3^r+c_3]^{tr}\mbox{ for }r=1\dots m.\eeq
 \end{theorem}
 %%%%%
 \begin{theorem} \label{main-affine-points}{(\sc affine projection criteria for lists of points.)} 
  A given list $Z=(\z^1,\dots, \z^m)$ of $m$ points 
 in $\RR^3$ with coordinates $\z^i=(z^r_1,z^r_2,z^r_3)$, $r=1\dots m$ projects onto a given list $X=(\x^1,\dots, \x^m)$  of  $m$ points in $\RR^2$  with coordinates $\x^r=(x^r,y^r)$  by an affine  projection if and only if there exist  $c_1, c_2\in \RR$, an  ordered triplet $(i,j,k)\in \left\{(1,2,3),\, (1,3,2), \,(2,3,1)\right\}$ and $[A]\in\aff(2)$, such that    
\beq\label{points-affine} [x^r,y^r,1]^{tr}=[A]\, \left[z^r_i+c_1 \,z^r_k,\,z^r_j+c_2\,z^r_k, 1\right]^{tr}\mbox{ for }r=1\dots m.\eeq
\end{theorem}
%%%%%
The proofs of Theorems~\ref{main-finite-points} and~\ref{main-affine-points} are straightforward adaptations of 
the proofs of Theorem~\ref{main-finite-camera} and~\ref{main-affine-camera}. 
The reduced affine projection criteria for curves, given in Corollary~\ref{reduced-aff-camera}, is  adapted  to the finite  lists in an analogous way.

The finite  and  the affine projection problems for lists of $m$  points are, therefore, reduced to  a modification of the problems of equivalence of two lists of  $m$ points in $\PP^2$ under the  action of $\pgl(3)$ and $\aff(2)$ groups, respectively. A separating set of invariants  for  lists of $m$ points in $\PP^2$ under $\aff(2)$-action consists of  ratios of certain areas and is listed, for instance, in Theorem~3.5 of \cite{olver01}. Similarly,  a separating set of invariants  for lists of $m$ ordered points in $\PP^2$  under $\pgl(3)$-action consists of cross-ratios of certain areas and is listed, for instance, in Theorem~3.10 in  \cite{olver01}.
In the case of finite projections, we, thus, obtain  a system of polynomial  equations on $c_1, c_2$ and $c_3$ that have solutions if and only if 
the given set $Z$ projects to the given set $X$. An analogue of Algorithm~\ref{alg-finite} for finite lists of points follows. The affine projections are treated in a similar  way.

 Figure~\ref{dc} illustrates that a solution of  the projection problem for lists of points does not provide  an immediate  solution to the discretization of the projection problem for curves. Indeed, 
if $Z=(\z^1,\dots, \z^m)$ is a discrete sampling of a curve $\Gamma$ and  $X=(\x^1,\dots, \x^m)$  is a discrete sampling of $\gamma$, these lists might not be in a correspondence under a projection even  when the curves are related by a projection.  Some approaches  to discretization of the projection algorithms for curves are discussed in the next section.

\begin{figure}\centering
\epsfig{file=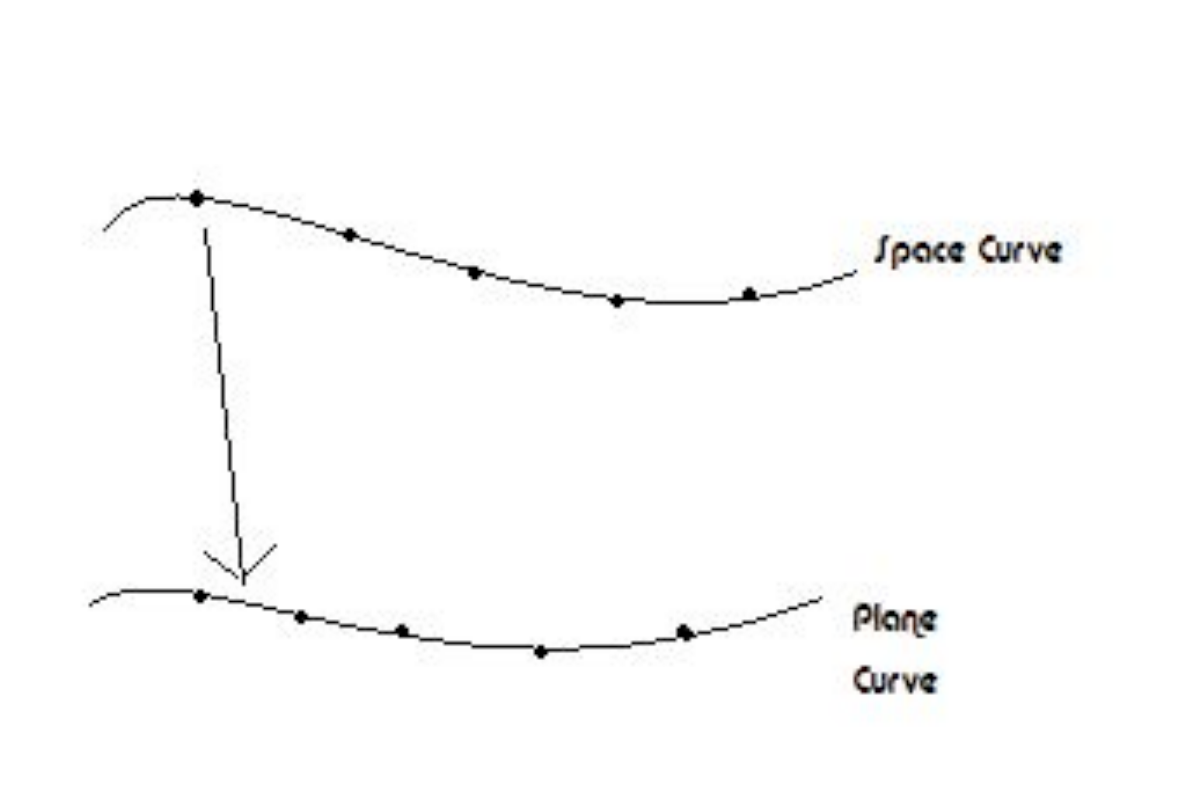,height=1in, width=1.5in}\caption{Projection problem for curves vs.~projection problems for  finite ordered sets of  points}\label{dc}\end{figure}
%

%%%%%%%%

\section{Directions of further research}\label{discussion}
 
The projection criteria developed in Section~\ref{criteria}  reduce the  problem of object-image correspondence for curves  under a projection from $\RR^3$ to $\RR^2$, to a variation of  the group-equivalence problem for curves in $\RR^2$. We  use differential signature construction \cite{calabi98} to 
address  the group-equivalence problem. In practical applications, curves are often presented by samples of  points.  In this case,  invariant numerical approximations of differential invariants presented in \cite{calabi98, boutin00} may be used to obtain signatures. Differential invariants and their approximations are highly sensitive to image perturbations and, therefore, are not practical in many situations.
Other types of invariants, such as semi-differential (or joint) invariants \cite{vang92-1, olver01}, integral invariants \cite{sato97, hann02, feng09} and moment invariants \cite{hu62} are less sensitive to image perturbations and may be employed to solve the group-equivalence problem. One of our future projects is to develop variations of Algorithms~\ref{alg-finite} and~\ref{alg-affine}  that are based on alternative solutions of the group-equivalence problem.

 One of the essential contributions of \cite{stiller06, stiller07} is the definition of an  object/image distance between ordered sets of $m$  points in   $\RR^3$ and $\RR^2$, such that the distance is zero   if and only if these sets are related by a projection.  Since, in practice, we are  given  only an approximate position of  points, a ``good'' object/image  distance   provides a tool for deciding whether a given set of points in $\RR^2$ is a good approximation of a  projection of a given set of points in $\RR^3$. Defining such object/image distance in the case of curves is an important direction of further research.

Although the projection algorithm presented here may not be immediatly applicable to real-life images, we consider this work  to be a first step toward  the development of more efficient algorithms to  determine projection correspondence for curves and other continuous  objects -- the problem  whose algorithmic solution, for classes of  projections with large degrees of freedom,   does not seem to appear in the literature.

\vskip3mm
\noindent{\bf Acknowledgements:} This project was inspired by a discussion with Peter Stiller of paper \cite{stiller07} during IS\&T/ SPIE 2007 symposium. We also thank Hoon Hong and Peter Olver for discussions of this project.\bibliographystyle{abbrv}
\bibliography{bib-11}

\def\cprime{$'$}
\begin{bibdiv}
\begin{biblist}

\bib{stiller06}{incollection}{
      author={Arnold, Gregory},
      author={Stiller, Peter~F.},
      author={Sturtz, Kirk},
       title={Object-image metrics for generalized weak perspective
  projection},
        date={2006},
   booktitle={Statistics and analysis of shapes},
      series={Model. Simul. Sci. Eng. Technol.},
   publisher={Birkh\"auser Boston},
     address={Boston, MA},
       pages={253\ndash 279},
}

\bib{stiller07}{inproceedings}{
      author={Arnold, Gregory},
      author={Stiller, Peter~F.},
       title={Mathematical aspects of shape analysis for object recognition},
        date={2007},
   booktitle={proceedings of is\&t/spie joint symposium},
     address={San Jose, CA},
       pages={11p},
}

\bib{boutin00}{article}{
      author={Boutin, M.},
       title={Numerically invariant signature curves},
        date={2000},
     journal={Int.~J.~Computer~Vision},
      volume={40},
       pages={235\ndash 248},
}

\bib{bk-prep}{article}{
      author={Burdis, J.~M.},
      author={A., Kogan~I.},
       title={{Object-image correspondence for central and parallel
  projections}},
        date={in preparation},
}

\bib{burdis10}{thesis}{
      author={Burdis, J.~M.},
       title={{Object-image correspondence under projections.}},
        type={Ph.D. Thesis},
        date={North Carolina State University, 2010},
}

\bib{calabi98}{article}{
      author={Calabi, E.},
      author={Olver, P.~J.},
      author={Shakiban, C.},
      author={Tannenbaum, A.},
      author={Haker, S.},
       title={Differential and numerically invariant signature curves applied
  to object recognition},
        date={1998},
     journal={Int.~J.~Computer~Vision},
      volume={26},
       pages={107\ndash 135},
}

\bib{C37}{book}{
      author={Cartan, \'{E}.},
       title={Groupes finis et continus et la ge\'om\'etrie diff\'erentielle
  trait\'ees par la methode du rep\`ere mobile.},
   publisher={Gauthier-Villars, Paris},
        date={1937},
}

\bib{Collins:75}{inproceedings}{
      author={Collins, G.~E.},
       title={Quantifier elimination for the elementary theory of real closed
  fields by cylindrical algebraic decomposition},
        date={1975},
   booktitle={Lecture notes in computer science},
   publisher={Springer-Verlag, Berlin},
       pages={134\ndash 183},
        note={Vol. 33},
}

\bib{CLO96}{book}{
      author={Cox, D.},
      author={Little, J.},
      author={O'Shea, D.},
       title={Ideals, varieties, and algorithms},
      series={Undergraduate Text in Mathematics},
   publisher={Springer-Verlag, New-York},
        date={1997},
}

\bib{faugeras94}{article}{
      author={Faugeras, O.},
       title={{Cartan's moving frame method and its application to the geometry
  and evolution of curves in the Euclidean, affine and projective planes}},
        date={1994},
     journal={Application of Invariance in Computer Vision, J.L Mundy, A.
  Zisserman, D. Forsyth (eds.) Springer-Verlag Lecture Notes in Computer
  Science},
      volume={825},
       pages={11\ndash 46},
}

\bib{feldmar95}{inproceedings}{
      author={Feldmar, J.},
      author={Ayache, N.},
      author={Betting, F.},
       title={{3D-2D} projective registration of free-form curves and
  surfaces},
        date={1995},
   booktitle={Proceedings of the fifth international conference on computer
  vision, {ICCV}, {IEEE Computer Society}},
     address={Washington, DC},
       pages={549 \ndash  556},
}

\bib{feng09}{article}{
      author={Feng, S.},
      author={Kogan, I.~A.},
      author={Krim, H.},
       title={{Classification of curves in 2D and 3D via affine integral
  signature}},
        date={2010},
     journal={Acta Appl. Math.},
      volume={109},
      number={3},
       pages={903\ndash 937},
}

\bib{fulton}{book}{
      author={Fulton, William},
       title={Algebraic curves},
      series={Advanced Book Classics},
   publisher={Addison-Wesley Publishing Company Advanced Book Program},
     address={Redwood City, CA},
        date={1989},
        note={An introduction to algebraic geometry, Notes written with the
  collaboration of Richard Weiss, Reprint of 1969 original},
}

\bib{Grig88}{article}{
      author={Grigoriev, D.},
       title={Complexity of deciding tarski algebra},
        date={1988},
     journal={J. Symb. Comput.},
      volume={5},
      number={1/2},
       pages={65\ndash 108},
}

\bib{Gug63}{book}{
      author={Guggenheimer, H.~W.},
       title={Differential geometry},
   publisher={McGraw-Hill, New York},
        date={1963},
}

\bib{hann02}{article}{
      author={Hann, C.},
      author={Hickman, M.},
       title={Projective curvature and integral invariants},
        date={2002},
     journal={Acta applicandae mathematicae},
      volume={74},
       pages={549 \ndash  556},
}

\bib{hartley04}{book}{
      author={Hartley, R.~I.},
      author={Zisserman, A.},
       title={Multiple view geometry in computer vision},
     edition={Second},
   publisher={Cambridge University Press},
        date={2004},
}

\bib{HRS93}{article}{
      author={Heintz, J.},
      author={Roy, M.-F.},
      author={Solern\`o, P.},
       title={On the theoretical and practical complexity of the existential
  theory of the reals},
        date={1993},
     journal={The Computer Journal},
      volume={36},
      number={5},
       pages={427\ndash 431},
}

\bib{Hong:90a}{inproceedings}{
      author={Hong, H.},
       title={An improvement of the projection operator in cylindrical
  algebraic decomposition},
        date={1990},
   booktitle={International symposium of symbolic and algebraic computation
  (issac-90)},
   publisher={ACM},
       pages={261\ndash 264},
}

\bib{hu62}{article}{
      author={Hu, M.~K.},
       title={Visual pattern recognition by moment invariants},
        date={1962},
     journal={IRE Trans. Info. Theory},
      volume={IT-8},
       pages={179\ndash 187},
}

\bib{kogan03}{article}{
      author={Kogan, I.~A.},
       title={Two algorithms for a moving frame construction},
        date={2003},
     journal={Canad. J. Math.},
      volume={55},
       pages={266\ndash 291},
}

\bib{musso09}{article}{
      author={Musso, E.},
      author={Nicolodi, L},
       title={Invariant signature of closed planar curves},
        date={2009},
     journal={J. Math. Imaging Vision},
      volume={35},
       pages={68\ndash 85},
}

\bib{olver01}{article}{
      author={Olver, P.~J.},
       title={Joint invariant signatures},
        date={2001},
     journal={Found. Comp. Math},
      volume={1},
       pages={3\ndash 67},
}

\bib{olver::inv}{book}{
      author={Olver, Peter~J.},
       title={Classical invariant theory},
   publisher={Cambridge University Press, Cambridge},
        date={1999},
}

\bib{pedersen93}{incollection}{
      author={Pedersen, P.},
      author={Roy, M.-F.},
      author={Szpirglas, A.},
       title={Counting real zeros in the multivariate case},
        date={1993},
   booktitle={Computational algebraic geometry ({N}ice, 1992)},
      series={Progr. Math.},
      volume={109},
   publisher={Birkh\"auser Boston},
     address={Boston, MA},
       pages={203\ndash 224},
}

\bib{sato97}{article}{
      author={Sato, J.},
      author={Cipolla, R.},
       title={Affine integral invariants for extracting symmetry axes},
        date={1997},
     journal={Image and Vision Computing},
      volume={15},
       pages={627\ndash 635},
}

\bib{tarski:51}{book}{
      author={Tarski, A.},
       title={A decision method for elementary algebra and geometry},
     edition={second},
   publisher={Univ. of California Press, Berkeley},
        date={1951},
}

\bib{vang92-1}{article}{
      author={Van~Gool, L.},
      author={Moons, T.},
      author={Pauwels, E.},
      author={Oosterlinck, A.},
       title={Semi-differential invariants},
        date={1992},
     journal={Geometric Invariance in Computer Vision, J.L. Mundy and A.
  Zisserman (eds), MIT Press},
       pages={157\ndash 192},
}

\bib{pinhole-figure}{misc}{
      author={Website},
        note={\tt http://en.wikivisual.com/index.php/Image:Pinhole-camera.png},
}

\end{biblist}
\end{bibdiv}
\end{document}